\documentclass{article}
\usepackage{ijcai24}

\usepackage{times}
\usepackage{soul}
\usepackage{url}
\usepackage[hidelinks]{hyperref}
\usepackage[utf8]{inputenc}
\usepackage[small]{caption}
\usepackage{graphicx}
\usepackage{amsmath}
\usepackage{amsthm}
\usepackage{booktabs}

\usepackage[switch]{lineno}

\usepackage{amsmath,amssymb,amsfonts}
\usepackage{textcomp}
\usepackage{color}

\usepackage{threeparttable}

\usepackage{algorithmic}
\usepackage[linesnumbered,boxed,ruled,commentsnumbered]{algorithm2e}

\newtheorem{theorem}{Theorem}

\SetKwInOut{Input}{\textbf{Input}}
\SetKwInOut{Output}{\textbf{Output}}
\SetKw{Continue}{continue}

\title{FlagVNE: A Flexible and Generalizable Reinforcement Learning  Framework \\ for Network Resource Allocation}

\author{
    Tianfu Wang$^1$
    \and
    Qilin Fan$^2$
    \and
    Chao Wang$^{3,1}$\thanks{Chao Wang and Hui Xiong are both corresponding authors}
    \and
    \\
    Long Yang$^4$ 
    \and
    Leilei Ding$^1$
    \and
    Nicholas Jing Yuan$^2$
    \And
    Hui Xiong$^{6,7}$$^*$
    \\
    \affiliations
    $^1$School of Computer Science and Technology, University of Science and Technology of China, China
    \\
    $^2$School of Big Data and Software Engineering, Chongqing University, China
    \quad
    $^3$Guangzhou HKUST Fok Ying Tung Research Institute, China
    \quad
    $^4$School of Artificial Intelligence, Peking University, China
    \\
    $^5$Microsoft
    \quad
    $^6$Thrust of Artificial Intelligence, The Hong Kong University of Science and Technology (Guangzhou), China
    \quad
    $^7$Department of Computer Science and Engineering, The Hong Kong University of Science and Technology, Hong Kong SAR, China
    \\
    \emails
    \{tianfuwang,dingleilei\}@mail.ustc.edu.cn, fanqilin@cqu.edu.cn,chadwang2012@gmail.com\\ yanglong001@pku.edu.cn,nicholas.yuan@microsoft.com,xionghui@ust.hk
}

\begin{document}

\maketitle
\begin{abstract}
Virtual network embedding (VNE) is an essential resource allocation task in network virtualization, aiming to map virtual network requests (VNRs) onto physical infrastructure. Reinforcement learning (RL) has recently emerged as a promising solution to this problem. However, existing RL-based VNE methods are limited by the unidirectional action design and one-size-fits-all training strategy, resulting in restricted searchability and generalizability. In this paper, we propose a \textbf{FL}exible \textbf{A}nd \textbf{G}eneralizable RL framework for \textbf{VNE}, named \textbf{FlagVNE}.  Specifically, we design a bidirectional action-based Markov decision process model that enables the joint selection of virtual and physical nodes, thus improving the exploration flexibility of solution space. To tackle the expansive and dynamic action space, we design a hierarchical decoder to generate adaptive action probability distributions and ensure high training efficiency. Furthermore, to overcome the generalization issue for varying VNR sizes, we propose a meta-RL-based training method with a curriculum scheduling strategy, facilitating specialized policy training for each VNR size. Finally, extensive experimental results show the effectiveness of FlagVNE across multiple key metrics. 
Our code is available at GitHub\footnote{\href{https://github.com/GeminiLight/flag-vne}{https://github.com/GeminiLight/flag-vne}}.
\end{abstract}
\section{Introduction}
Network virtualization (NV) emerges as a pioneering technology that facilitates dynamic management of Internet architecture, which finds applications in 5G networks and cloud computing~\cite{vne-nfv-sdn}. Through network slicing and shared infrastructure, NV enables the deployment of multiple user-submitted virtual network requests (VNRs) within the same physical network, thereby accommodating diverse network service requirements of users \cite{vne-nfv,chen2022entity}.
The primary challenge in NV involves the embedding of VNRs within a physical network, known as virtual network embedding (VNE), an NP-hard combinatorial optimization problem~\cite{vne-nphard}.

Effective resource allocation for VNRs is essential to improve the quality of service and the revenue of Internet service providers (ISPs)~\cite{wang2021personalized,liwei-kdd22}.
Regrettably, it is hard to address the VNE problem involving tackling combinatorial explosion and differentiated demands~\cite{vne-survey,yang-aaai22}.
On the one hand, the solution space of VNE is extensive, encompassing vast permutations of VNRs within the underlying physical network. Consequently, a comprehensive exploration of this expansive solution space becomes imperative to determine superior solutions.
On the other hand, due to specific requirements of user service, the integration of diverse VNR topologies and their associated resource demands is dynamic.
VNRs of varying sizes manifest unique complexities, rendering a one-size-fits-all strategy inadequate to effectively manage the inherent variability in such circumstances.

Recently, reinforcement learning (RL) has shown promising potential for the VNE problem~\cite{vne-jsac-2020-a3c-gcn,vne-tpds-2023-att,vne-tsc-2023-gcn-mask}. RL approaches model the solution construction process of each VNR as Markov decision processes (MDPs), which can automatically build efficient solving policies. Unlike supervised learning relying on labeled data, RL facilitates the learning of effective heuristics through interactions with the environment. 
However, most existing RL-based VNE approaches are still plagued with some significant issues. 
Firstly, these approaches commonly adhere to a unidirectional action design within the MDP, i.e., presupposing a fixed decision sequence for virtual nodes, and subsequently designating a physical node to host each virtual node sequentially. Such unidirectional action schema significantly limits the available action space, consequently constraining the searchability of the agent and impeding the efficacy of exploring solution space.
Secondly, conventional RL-based methods usually just train a single general policy, disregarding the distinctive complexities of VNRs with varying sizes in practice. Treating variable-sized VNRs equally poses challenges in achieving balanced learning of cross-size strategic knowledge and hinders the ability to generalize across VNRs of differing sizes.
Thirdly, the direct training of multiple policies tailored to different VNR sizes slowly adapts to unseen distributions. In particular, training specific policies for large-sized VNRs from scratch tends to be stuck in the local optimum, due to the high complexity and challenges in exploring feasible solutions.
These difficulties inevitably exert negative impacts on overall system performance. We conduct a preliminary study to highlight our motivations and latent challenges, which is detailed in Appendix~\ref{appendix:preliminary-study}.

In this paper, we propose a novel \textbf{FL}exible \textbf{A}nd \textbf{G}eneralizable RL framework for the \textbf{VNE} problem, named \textbf{FlagVNE}.
Our framework aims to enhance the searchability and generalizability of RL-based VNE methods while achieving rapid adaption to the unseen distribution of VNR sizes. 
Specifically, our contributions are summarized as follows. (1) We propose a bidirectional action-based MDP modeling approach to enable the joint selection of virtual and physical nodes, enhancing the flexibility of agent exploration and exploitation. This method offers superior searchability and is proven theoretically. To handle the resulting large and changeable action space, we abstract it as two dependent aspects and design a hierarchical decoder with a bilevel policy, ensuring adaptive action probability distribution generation and high training efficiency.
(2) We propose a meta-RL-based training method to enable efficient acquisition of multiple size-specific policies and quick adaptation to new sizes. A meta-policy is trained to grasp cross-size knowledge for different VNR sizes and then fastly fine-tuned to develop size-specific policies for each VNR size, even unseen sizes. 
Specially, due to difficult exploration and prone to suboptimal convergence, using large-sized VNRs for initial meta-learning yields inferior knowledge, impairing the meta-policy and generalization. Thus, we develop a curriculum scheduling strategy that gradually incorporates larger VNRs, alleviating suboptimal convergence.
(3) Finally, we conduct experiments on the simulation platform to mimic various network systems and extensive results demonstrate the superiority of FlagVNE in terms of multiple key indicators, compared to state-of-the-art (SOTA) heuristics and RL-based methods.
\section{Related Work}
\label{section:background-and-motivation}
In this section, we review the related work on VNE.

\textbf{Traditional Methods for VNE.} Initially, the VNE problem was tackled using exact methods such as integer linear programming \cite{vne-jsac-2018-ilp}, which provides optimal solutions through exact solvers. However, these exact algorithms proved impractical for real-world scenarios due to their time-consuming nature. Thus, numerous heuristic algorithms have been proposed to find solutions in an acceptable time~\cite{vne-ton-2014-pso,vne-infocom-2020-latency,vne-tpds-2023-nea}. Among these approaches, node ranking is a prevalent strategy, which ranks virtual and physical nodes to determine the decision sequence and the matching priority, respectively. For example, \cite{vne-iotj-2018-nrm} ranked nodes based on a node resource management (NRM) metric, and \cite{vne-tpds-2023-nea} proposed a node essentiality assessment (NEA) metric considering topology connectivity. Additionally, \cite{vne-jsac-2019-fitness} designed VNE algorithms based on metaheuristics, such as particle swarm optimization (PSO). However, these algorithms heavily rely on manual designs and are usually tailored to specific scenarios, limiting their performance in general cases.


\textbf{Learning-based Methods for VNE.}
Recently, machine learning techniques have been used to solve VNE, leading to faster and more efficient solutions~\cite{vne-infocom-2018-neuro,kdd-2023-gal-vne,vne-tpds-2023-att}. 
Particularly, RL has demonstrated significant potential as an intelligent decision-making framework~\cite{liwei-icde23,yang-nips22}, which can effectively solve VNE with MDP modeling. In this paper, we unify most existing RL-based methods~\cite{vne-iwqos-2019-mlp,vne-icc-2021-drl-sfcp,vne-jsac-2020-a3c-gcn,vne-tnsm-2020-pg-rnn,vne-tnsm-2022-pg-cnn,vne-tsc-2023-gcn-mask} 
into a general framework comprised of three key components: \textit{MDP modeling}, \textit{policy architecture}, and \textit{training methods}. These methods model the process of VNE solution construction as unidirectional action-based MDPs, where a physical node is chosen to host a be-placing virtual node, and the decision sequence of virtual nodes is fixed. Then they build policy models with various neural networks and train a single general policy to deal with VNRs of varying sizes. For instance, \cite{vne-iwqos-2019-mlp} used multilayer perception (MLP) as a policy model and trained it with policy gradient (PG) algorithm, \cite{vne-tsc-2023-gcn-mask} designed a policy model with MLP and graph convolutional network (GCN) \cite{gnn-gcn} and trained it with asynchronous advantage actor-critic (A3C) \cite{rl-icml-2016-a3c}.
However, existing RL-based VNE methods suffer from limited searchability and generalizability due to their unidirectional action design and one-size-fits-all training policy, ultimately affecting overall performance.

\section{Preliminaries}

\subsection{Problem Definition}

\begin{figure}[t]
    \centering
    \includegraphics[width=.42\textwidth]{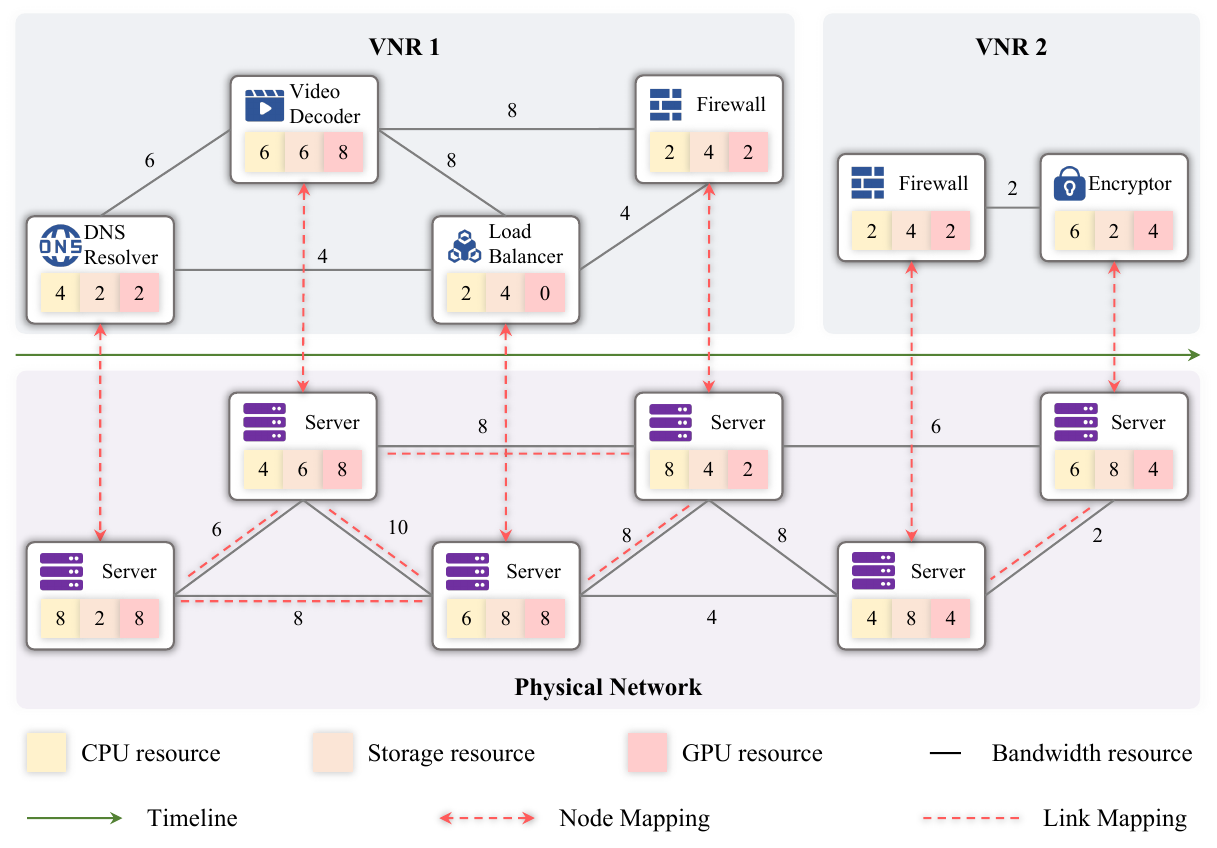}
    \vspace{-1em}
    \caption{An example of the VNE problem with multidimensional resources. The numbers denote the unit counts of resources.}
    \label{fig:vne-example}
    \vspace{-1.4em}
\end{figure}

As shown in Fig.~\ref{fig:vne-example}, in a practical network system, users' service requests arriving continuously are represented as VNRs. We collect all VNRs with a set $\mathcal{V}$. Mapping these VNRs onto physical networks managed by ISPs is known as VNE, crucial in managing the quality of various network services~\cite{chen2020mmea,wang2023setrank,chen2022msnea}.

\textbf{System Modeling}. \textit{Physical network} is formulated as a weighted undirected graph $\mathcal{G}^p = (\mathcal{N}^p, \mathcal{L}^p)$, where $\mathcal{N}^p$ is the set of physical nodes, $\mathcal{L}^p$ is the set of physical links.
Each physical node $n^p \in \mathcal{N}^p$ is equipped with multiple resource capacities $\{C(n^p), \forall C\in\mathcal{C}\}$, where $\mathcal{C}$ is the set of node resource types, and each physical link $l^p \in \mathcal{L}^p$ has bandwidth capacity $B(l^p)$.
In this paper, we consider multidimensional node resources, including the central processing unit (CPU), storage resource, and graphics processing unit (GPU).
Similarly, each \textit{VNR} is modeled as a weighted undirected graph $G^v = (\mathcal{N}^v, \mathcal{L}^v, d^v)$, where $\mathcal{N}^v$ is the set of virtual nodes and $\mathcal{L}^v$ is the set of virtual links, and $d^v$ denotes the lifetime of VNR. Once the VNR is accepted, it will be maintained for $d^v$ time slots. Each virtual node $n^v \in \mathcal{N}^v$ represents a virtual machine with resource demands $\{C(n^v), \forall C \in \mathcal{C}\}$ and each virtual link $l^v \in \mathcal{L}^v$ indicates the bandwidth demand $B(l^v)$. 

\textbf{Objective.} Acknowledging the stochastic nature of online networking, most existing methods and this work aim to minimize the embedding cost of each VNR onto the physical network, which facilitates long-term performance. 
The quality of solutions is assessed using the revenue-to-cost ratio (R2C):
\begin{equation}
    \text{R2C} \left(\mathcal{G}^v\right) = \left(\Psi \cdot \text{REV} \left(\mathcal{G}^v\right) \right) / \text{COST}\left(\mathcal{G}^v\right).
\end{equation}
Here, $\text{REV} (\mathcal{G}^v)$ denotes the revenue of the VNR $G^v$ (i.e., the sum of VNR's resource requirements) and $\text{COST} (\mathcal{G}^v)$ denotes the embedding cost resulting from the solution (i.e., the sum of ISP' resource consumption). $\Psi$ is the binary variable that indicates the feasibility of a solution.

\textbf{Constraints.} The VNR embedding consists of two sub-processes. (1) \textit{Node mapping} entails assigning each virtual node to a physical node with adequate resources, i.e., $C(n^p) \geq C(n^v), \forall C \in \mathcal{C}$, while ensuring one-to-one placement and mutual exclusivity. (2) \textit{Link mapping} involves finding a physical path for each virtual link, ensuring that the path connects the physical nodes hosting the virtual link endpoints and that each physical link $l^p$ in the path has sufficient bandwidth, i.e., $B(l^p) \geq B(l^v)$. A solution is deemed feasible ($\Psi = 1$) only when all these constraints are satisfied.

Due to the space limit, we place detailed formulations of VNE's objective and constraints in Appendix~\ref{appendix:problem-formulation}. 

\subsection{Motivations and Challenges}
We conduct a preliminary study placed in Appendix~\ref{appendix:preliminary-study}, and 
motivate our framework from the following two aspects.

\textbf{Flexibility of Action Space.} 
Most existing RL-based VNE approaches employ a unidirectional action design, assuming that the decision sequence of virtual nodes is predetermined.
However, our analysis in Appendix~\ref{appendix:flexibility-of-action-space} reveals that varying the decision sequences of virtual nodes significantly impacts performance. This underscores the necessity of exploring different decision sequences for optimal solutions. Moreover, the fixed decision sequence of virtual nodes lacks the flexibility needed to adapt to the dynamic nature of exploration process. Thus,
to enhance the flexibility of exploration and exploitation,
we aim to achieve a joint selection of both physical and virtual nodes to eliminate the fixed decision sequence.
Nevertheless, it will pose some challenges, such as the difficulty of variable action probability distribution generation and the training efficiency issue caused by large action space.

\textbf{Generalization of Solving Policy.} 
VNRs of different sizes exhibit distinct complexities, necessitating varied solving strategies. Existing RL-based methods typically use a one-size-fits-all policy to tackle VNRs of varying sizes, leading to generalization issues. To address this, an intuitive approach might be to develop size-specific policies for different VNR sizes. Yet, as observed in Appendix~\ref{appendix:generalization-of-solving-policy}, specific policies for large-sized VNRs trained from scratch often get stuck in local optima due to their high complexity and the difficulty in exploring viable solutions. 
Their performance is even inferior to that of the general policy for all sizes.
Furthermore, this strategy lacks the quick adaptability to handle previously unseen VNR sizes, since it requires extensive data demand.

\section{FlagVNE Framework}
\begin{figure*}[t]
    \centering
    \includegraphics[width=.92\textwidth]{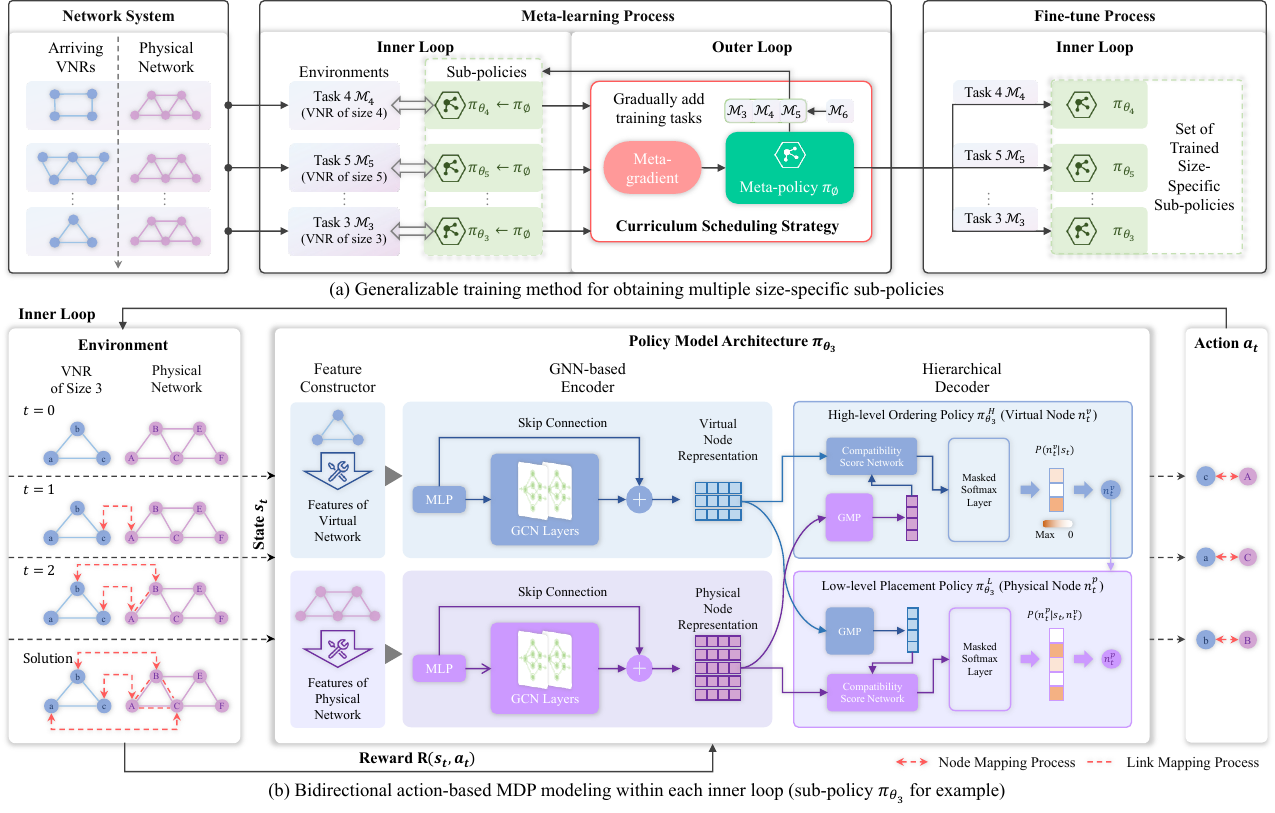}
    \vspace{-0.5em}
    \caption{The overview of the FlagVNE framework.
    (a) For vary-sized VNRs that continuously arrive at the network system, we consider them as different tasks $\mathcal{M}_i \sim p(\mathcal{M})$ based on their size. We first train a meta-policy $\pi_{\phi}$ with cross-task knowledge in the meta-learning process, using a curriculum scheduling strategy. Then, we fine-tune it to obtain a set of size-specific sub-policies $\pi_{{\theta}_i}$. This generalizable training method effectively obtains refined solving policies for each VNR size. (b) Within each inner loop, we formulate the solution construction process of each VNR as a bidirectional action-based MDP, which enables the joint selection of virtual and physical nodes. We also design a hierarchical encoder with a bilevel policy to adaptively generate action probability distributions and ensure high training efficiency.
   }
    \label{fig:mrl-framework}
    \vspace{-1.0em}
\end{figure*}

In this section, we present the proposed RL-based framework for VNE, \textbf{FlagVNE}. As illustrated in Fig.~\ref{fig:mrl-framework}, FlagVNE is designed to improve searchability and generalizability while achieving rapid adaptation to unseen distribution.

\subsection{Bidirectional Action-based MDP}
We formulate the solution construction process of each VNR as a bidirectional action-based MDP, allowing joint selection of virtual nodes and physical nodes. Specifically, at each decision timestep $t$, observing the state $s_t$ of the environment, the agent takes an action $a_t \sim \pi (\cdot | s_t)$ according to the policy $\pi$. Then, the environment will feedback a reward $R(s_t, a_t)$ and transit to a new state $s_{t+1} \sim P(s_t, a_t)$ following the transition probability function. During interactions, a trajectory memory $\mathcal{D} = \{s_1, a_1, s_2, a_2, \cdots\}$ collects state-action pairs. 
We present these notations in VNE as follows.

\textit{State} represents the status of the network system at a specific decision timestep $t$, consisting of the current situation of VNR $s_t^v$ and physical network $s_t^p$, i.e., $s_t = (s_t^v, s_t^p), s_t \in \mathcal{S}$, where $\mathcal{S}$ is the state space. 

\textit{Action} is defined as a pair of a virtual node to be placed and a physical node to host, denoted $a_t = (n^v, n^p), a_t \in \mathcal{A}$, where $n^v \in \mathcal{N}^v$, $n^p \in \mathcal{N}^p$, and $\mathcal{A}$ is the action space. 

\textit{Transition} $P(s_{t+1} \mid s_t, a_t)$ refers to the process of placing the virtual node $n^v$ to the physical node $n^p$ and routing the virtual links, resulting in the changes of state from $s_t$ to $s_{t+1}$. Based on the selected bidirectional action, the environment attempts to place the virtual node $n^v$ to the physical node $n^p$. If the node placement is successful, the link routing is executed based on the breadth-first search algorithm that finds the shortest physical paths meeting bandwidth demands from $n^p$ to other physical nodes hosting the virtual node neighbors of $n^v$. If node placement and link routing are successful, the available resource of the physical network is updated with the VNR requirement. Otherwise, the current VNR is rejected.

\textit{Reward} $R$ measures the quality of agent's action at a given state. We define the reward  function $R$ as follows:
\begin{equation}
    R(s_t, a_t) = \begin{cases}
        \text{R2C} (\mathcal{G}^v), & \text{if $\mathcal{G}^v$ is accepted at $t$}, \\
        -1 / |\mathcal{N}^v|, & \text{if $\mathcal{G}^v$ is rejected at $t$}, \\
       1 / |\mathcal{N}^v|, & \text{otherwise}.
    \end{cases}
\end{equation}
We design implicit rewards to encourage successful placement with $1 / |\mathcal{N}^v|$ and punish failure with $-1 / |\mathcal{N}^v|$. Once the $\mathcal{G}^v$ is completed embedding, we return $\text{R2C}(\mathcal{G}^v)$.

\textit{Policy} is parameterized by $\theta$, which denotes the distributions over the action space under a given state $s_t$:
\begin{equation}
    \pi_{\theta}(a_t | s_t) = P(a_t | s_t).
\end{equation}

\textit{Discount factor} $\lambda \in (0, 1)$ balances the importance of immediate rewards versus future rewards. Overall, the optimization objective of RL is to maximize the expected return, i.e., cumulative discounted rewards over timesteps $T$:
\begin{equation}
    J_{\pi} = \mathbb{E}_{(s_t, a_t) \sim \mathcal{D}} [ \sum_{t=0}^{T} \lambda^t R(s_t, a_t)].
\end{equation}

If $J_{\pi} \ge J_{\pi^{\prime}}$, then we denote it as $\pi \succeq \pi^{\prime}$.

\begin{theorem}
Given two MDPs with bidirectional and unidirectional action, $\mathcal{M}^b = \langle \mathcal{S}^b, \mathcal{A}^b, P^b, R, \lambda \rangle$ and $\mathcal{M}^u = \langle\mathcal{S}^u, \mathcal{A}^u, {P}^u, {R}, \lambda\rangle$, and their optimal policies denoted as $\pi^{\star, b}$ and $\pi^{\star, u}$, respectively, we have $\pi^{\star, b} \succeq \pi^{\star, u}$.
\label{theorem:mdp}
\end{theorem}

See Appendix~\ref{appendix:proof} for its proof. Our bidirectional action enhances flexibility and significantly expands the search space, allowing for a more comprehensive exploration of possible solutions, which offers superior MDP Optimality.

\subsection{Hierarchical Policy Architecture}
We construct raw features, encode them with a GCN-based encoder, and design a hierarchical decision module to ensure adaptive probability output and training efficiency.

\textbf{Feature Constructor.} 
We build the feature input for the subsequent encoder from the current state $s_t$, which includes the processing status of VNR $s^v_t$ and the current physical network situation $s^p_t$.
With comprehensive information on the current state, the agent gains deeper environmental insight, resulting in better decisions.
For the VNR $G^v_t$, the feature constructor takes into account not only various node resource requirements denoted $X^v_{t, N}$, but also aggregates bandwidth resource requirements into the node features, represented as $X^v_{t, L}$.
These features include essential bandwidth metrics, such as maximum, mean, and sum of bandwidth requirements of virtual links adjacent to one node.
To indicate the embedding status, a placement flag $X^v_{t,P}$ is designed for virtual nodes, with a value of 1 indicating that the virtual node has been placed and 0 otherwise.
The VNR features $X^v_t$ are organized as follows: $X^v_t = (X^v_{t, N}, X^v_{t, L}, X^v_{t, P}) \in \mathbb{R}^{|N^v| \times 7}.$

Similarly, the physical network features $X^p_t$ are constructed as follows: $X^p_t = (X^p_{t, N}, X^p_{t, L}, X^p_{t, S}) \in \mathbb{R}^{|N^p| \times 7},$ where $X^p_{t, N}$ denotes available resources of physical nodes, $X^p_{t, L}$ similar to $X^v_{t, L}$ denotes aggregated bandwidth availability, and $X^p_{t, S}$ is a selection flag indicating the status of physical nodes.
A selection flag value of 1 indicates that a physical node has been selected to host a virtual node and 0 otherwise.

\textbf{GNN-based Encoder.} To encode the features of the virtual network $X^v_t$ and the physical network $X^p_t$, into latent representations, $Z^v_t$ and $Z^p_t$, respectively, we adopt a graph neural network (GNN) encoder.
First, both $X^v_t$ and $X^p_t$ undergo the MLP to obtain the initial node representations, denoted $I^v_t$ and $I^p_t$, respectively: $I^v_t = \text{MLP}(X^v_t), I^p_t = \text{MLP}(X^p_t).$

Then, we consider multiple GCN \cite{gnn-gcn} layers as the GNN modules to obtain the latent representations of virtual nodes $\tilde{Z}^v_t$ and physical nodes $\tilde{Z}^p_t$: $\tilde{Z}^v_t = \text{GNN}(I^v_t, A^v), \tilde{Z}^p_t = \text{GNN}(I^p_t, A^p)$, where $A^v$ and $A^p$ is adjacency matrixes of virtual and physical networks. 

Furthermore, to enhance the feature representation ability, we also employ the residual connection method to combine the output of the GNN module with the initial representation: $Z^v_t = \tilde{Z}^v_t + I^v_t, Z^p_t = \tilde{Z}^p_t + I^p_t$.
Finally, we obtain the representation of each virtual and physical node.

\textbf{Hierarchical Decoder with Bilevel Policy.}
In our bidirectional action-based MDP, the action space is represented by the matrix size $|\mathcal{N}^v| \times |\mathcal{N}^p|$, reflecting the number of virtual and physical nodes. The variable and often large size of VNRs contribute to the expansive and dynamic nature of the space. To effectively manage this, we develop a hierarchical decoder with a bilevel policy, ensuring high training efficiency and adaptive action probability generation. Specifically, we abstract this task into two dependent aspects: virtual node ordering and physical node placement. Our bilevel policy, $\pi (a_t | s_t) = \pi^H(n^v | s_t) \cdot \pi^L(n^p | s_t, n^v)$, consist of a high-level ordering policy $\pi^H(n^v | s_t)$ and a low-level placement policy $\pi^L(n^p | s_t, n^v)$. This hierarchical approach reduces the size of policy distribution from $|\mathcal{N}^v| \times |\mathcal{N}^p|$ to $|\mathcal{N}^v| + |\mathcal{N}^p|$, thus significantly enhancing training efficiency.

\textit{High-level ordering policy} selects the appropriate virtual node $n^v_t$ for placement. Concretely, we use an MLP-based compatibility scoring network to calculate the fitness between each virtual node representation and the graph-level representation of the physical network $G^p_t = \text{GMP}(Z_t^p)$. Here, $\text{GMP}(Z) = \frac{1}{|Z|} \sum_{z \in Z} z$  denotes graph mean pooling (GMP), averaging all node representations. Then an MLP is applied to generate compatibility scores for each virtual node:
\begin{equation}
    \tilde{Y}^H = \mathrm{MLP}(Z^v_t + G^p_t) \in \mathbb{R}^{1 \times |\mathcal{N}^v|}.
\end{equation}
Although the VNR's sizes are variable, this layer adaptively generates scores with the shape of $(1, |\mathcal{N}^v|)$.
After masking virtual nodes already placed (i.e., setting their scores to $-\infty$ on $\tilde{Y}^H$), we apply a softmax function to the resultant score $Y^H$ to produce the high-level action probability distribution.
\begin{equation}
    \pi^H(n_t^v | s_t) = \mathrm{softmax}(Y_H).
\end{equation}

\textit{Low-level placement policy} identifies a suitable physical node $n^p_t$ for accommodating the to-be-placed virtual node $n^v_t$, which is selected by $\pi^H$. Similarly, we adopt an MLP-based compatibility scoring network to calculate the fitness between the representation of each physical node and the current context representation of virtual network, including the graph-level representation of virtual network $G^v_t = \text{GMP}(Z^v_t)$ and to-be-placed virtual node's representation $z_{n^v_t}$:
\begin{equation}
    \tilde{Y}^L = \mathrm{MLP}(Z^p_t + G^v_t + z_{n^v_t}) \in \mathbb{R}^{1 \times |\mathcal{N}^p|}.
\end{equation}
To avoid unnecessary exploration, we mask the physical nodes that do not have enough resources or have been selected to obtain the final scores $Y^L$. Then, the low-level action probability distribution is generated:
\begin{equation}
    \pi^L(n^p_t | s_t, n^v_t) = \mathrm{softmax}(Y_L).
\end{equation}

For both two-level probability distributions, we employ the sampling and greedy strategy to select actions during the training and inference phases, respectively.

\subsection{Generalizable Training Method}

Training a general policy for VNRs of varying sizes leads to imbalanced learning of cross-size strategy and generalization issues. Conversely, individualized training of multiple policies for each size is slow to adapt to new sizes, in which policies for large-sized VNRs are prone to suboptimal. To address this, we develop a meta-RL-based training method with a curriculum scheduling strategy. As illustrated in Algorithm~\ref{algo:flagvne-training} (see Appendix~\ref{appendix-pseudocode}), our method enables efficient training of multiple size-specific policies and quick adaptation to new sizes, while balancing the learning process across tasks of varying difficulty and avoiding suboptimal convergence.

\textbf{Meta-RL for VNE.} We treat VNRs of different sizes as distinct tasks and formulate them as multiple MDPs following a distribution $\mathcal{M}_i \sim p(\mathcal{M})$. Note that this distribution of VNR size is bounded and always obviously smaller than the number of physical nodes, following the network service orchestration standards~\cite{vne-nfv-sdn}.
We adopt model-agnostic meta-learning (MAML) as the basic training method~\cite{dl-icml-2017-maml}. 
MAML facilitates the learning of a meta-policy that can be swiftly fine-tuned on new tasks with only a few training samples, which improves generalizability and adaptability.
This training process comprises two stages as follows. Firstly, during the meta-learning process, we iteratively execute the inner loops and outer loops to derive a well-trained meta-policy $\pi_{\phi}$ with cross-task knowledge. Secondly, in the fine-tuning process, we leverage task-specific experiences to fine-tune the meta-policy to a set of size-specific policies $\pi_{{\theta}_i}$ solely through inner loops. 

Concretely, in the inner loop, the meta-policy $\pi_\phi$ is updated to accommodate a specific task $\mathcal{M}_i$ by performing gradient descents with the learning rate $\alpha$ and task-specific data $\mathcal{D}_i$:
\begin{equation}
\theta_i = f\left( \phi, \mathcal{D}_i\right) = \phi - \alpha \nabla_{\phi} \mathcal{L}_{\mathcal{D}_i}({\phi}).
\label{eq:maml-inner}
\end{equation}
Here, $\mathcal{L}(\cdot)$ follows the objective of proximal policy optimization (PPO) algorithm~\cite{drl-ppo}:
\begin{equation}
\mathcal{L}_{\mathcal{D}_i}(\phi)=\mathbb{E}_{(s_t, a_t) \sim \mathcal{D}_i}\left[\min \left(r_{\phi} \hat{A}, \operatorname{clip} \left( r_{\phi}, \epsilon\right) \hat{A})\right)\right],
\label{eq:ppo}
\end{equation}
where $\hat{A}$ denotes the estimated advantage of taking an action. $r_{\phi} = \frac{\pi_{\phi} (a_t | s_t)}{\pi_{\phi_{\text{old}}} (a_t | s_t)}$ denotes the ratio between the current policy $\pi_{\phi}$ and the last updated policy $\pi_{\phi_{\text{old}}}$. The clip function with a hyperparameter $\epsilon$ is used to limit $r_{\phi}$ within the range of $[1-\epsilon, 1+\epsilon]$, improving the stability of policy updates. 
In PPO, the critic uses a GNN-based encoder and GMPs, then inputs concatenated virtual and physical graph representations into an MLP-based decoder to estimate value.

In the outer loop, our objective is to find a meta-policy $\pi_{\phi}$ that learns balanced strategy knowledge required by VNRs of different sizes and exhibits superior generalizability, enabling it to quickly learn optimal task-specific policies:
\begin{equation}
    J_\phi = \mathbb{E}_{\mathcal{M}_i \sim p(\mathcal{M})} \left[ \mathbb{E} \left[ \sum\nolimits_{t=0}^{T} \lambda^t R(s_t, a_t) | \theta_{i}, \mathcal{D}_i \right] \right].
\end{equation}

We update $\phi$ with a meta-learning rate $\beta$ according to average second-order meta-gradient over task-specific policies:
\begin{equation}
\phi \leftarrow \phi - \beta \nabla_{\phi}
 \left(\frac{1}{|\mathcal{M}|} \sum\nolimits_{i = 1}^{|\mathcal{M}|} \mathcal{L}({\theta_{i}}) \right).
\label{eq:maml-outer}
\end{equation}

\begin{figure*}[t]
    \centering
    \includegraphics[width=1.00\textwidth]{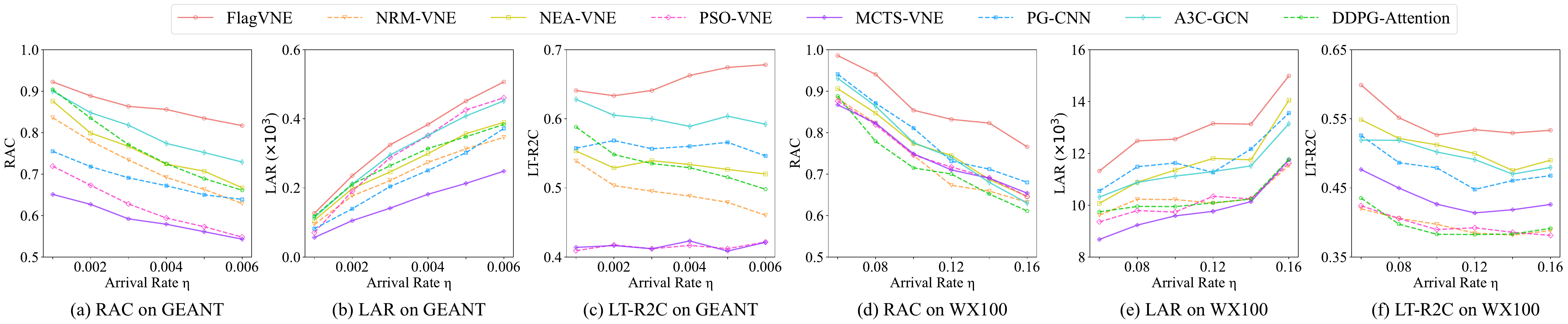}
    \vspace{-1.34em}
    \caption{Experimental results in traffic throughput test. } 
    \label{fig:traffic-throughput-test}
\end{figure*}

\textbf{Curriculum Scheduling Strategy.}  In our preliminary study (see Appendix~\ref{appendix:generalization-of-solving-policy}), we observed that training specific policies for large VNRs often leads to suboptimal convergence. This issue stems from the complexity of large-sized VNRs and the challenges of exploring the solution space to find feasible solutions. This tendency also towards local optima adversely impacts the meta-learning process. Specifically, using large-sized VNRs in the initial stages of meta-learning results in low-quality gradients, which negatively affects the convergence and generalizability of the meta-policy.

To address this challenge, we draw inspiration from curriculum learning~\cite{tpami-2021-curriculum-learning-survey} and propose a curriculum scheduling strategy to gradually integrate larger VNRs into the meta-learning process. 
This strategy enables high-quality initializations for sub-policies of large-sized VNRs, alleviating the problems of suboptimal convergence and compromising meta-policy.
We implement this by maintaining a training task list 
$\mathcal{I}$, initially containing the smallest VNR size.
The meta-learning process begins by focusing on tasks with smaller VNR sizes, which are inherently easier and provide beneficial foundational knowledge for tackling more complex tasks. Policies adeptly trained on these smaller tasks serve as effective initializations for larger VNR tasks, facilitating to mitigating local optima issues.

To achieve a gradual increase in task complexity, we use the entropy metric $H(\pi)$ to evaluate the stability of policy.  For our bilevel policy, we approximate it with $H(\pi) = H(\pi^H) + H(\pi^L)$. 
A lower entropy suggests that the policy is making more confident decisions.
When the policy entropy $H(\pi_{\theta_k})$ for the largest size $k = \mathrm{max}(\mathcal{I})$ currently on the training task list falls below a specified threshold $\delta$, we consider the policy ready to handle more complex tasks. At this point, we introduce the next larger VNR size to the training task list $\mathcal{I}$. This progressive approach allows the meta-policy to adapt and generalize effectively to larger VNRs.

\section{Performance Evaluation}
\label{section:evaluation}
In this section, we evaluate the effectiveness of FlagVNE. 

\begin{table*}[]
\centering
\begin{threeparttable}
\begin{tabular}{c|ccc|ccc}
\toprule
& \multicolumn{3}{c|}{GEANT} & \multicolumn{3}{c}{WX100} \\ \cline{2-7}
& RAC $\uparrow$ & LAR $\uparrow$ & LT-R2C $\uparrow$ & RAC $\uparrow$ & LAR $\uparrow$ & LT-R2C $\uparrow$ \\ 
\midrule
FlagVNE-UniActionNEA & 0.781 & 475.335 & 0.637 & 0.724 & 14334.671 & 0.493 \\
FlagVNE-MetaFree-SinglePolicy & 0.758 & 472.455 & 0.614 & 0.712 & 14170.514 & 0.501 \\
FlagVNE-MetaFree-MultiPolicy & 0.746 & 435.502 & 0.593 & 0.685 & 14069.938 & 0.472 \\
FlagVNE-MetaPolicy & 0.773 & 478.646 & 0.634 & 0.717 & 14292.962 & 0.485 \\
FlagVNE-NoCurriculum & 0.787 & 485.267 & 0.643 & 0.708 & 14144.234 & 0.509 \\ \midrule
\textbf{FlagVNE} & \textbf{0.804} & \textbf{499.303} & \textbf{0.668} & \textbf{0.754} & \textbf{14769.080} & \textbf{0.526} \\ 
\bottomrule
\end{tabular}
\end{threeparttable}
\caption{Results on ablation study. ($\eta = 0.006$ on GEANT and $\eta = 0.18$ on WX100).}
\label{table:ablation-study}
\vspace{-1.em}
\end{table*}

\subsection{Experiment Setup}
\label{section:experiment-setup}

\textbf{Simulations}. Following the latest works~\cite{vne-tpds-2023-att,tsc-2023-hrl-acra}, we conduct experiments on the simulation platform to mimic various realistic network systems. We adopt two topologies, GEANT (40 nodes and 61 links) and WX100 (100 nodes and 500 links)~\cite{dataset-jsac-1988-waxman}, as physical networks. See Appendix~\ref{appendix:topology} for these topologies' descriptions. The multiple-type resources (i.e., CPU, storage, GPU) of physical nodes and bandwidth resources of physical links are uniformly generated within the range of [50, 100] units. In each simulation run, we randomly generate 1000 VNRs with varying sizes ranging from 2 to 10. The virtual nodes within each VNR are randomly interconnected with a probability of 50\%. Additionally, resource demands of each VNR's node and link requirements are uniformly generated within the range of [0, 20] and [0, 50] units, respectively. The lifetime of each VNR is exponentially distributed with an average of 500 time units. The arrival of these VNRs follows a Poisson process with an average rate $\eta$, wherein $\eta$ VNRs are received per unit of time. In subsequent experiments, we first train models with $\eta = 0.001$ on GEANT and $\eta = 0.08$ on WX100, due to their different capacities of physical resources. 
Then we manipulate the value of $\eta$ to emulate network systems with different traffic throughputs and infer with trained models to study the sensitivity of algorithms.

\textbf{Implementations}. During training, we first conduct meta-learning in the initial 20 simulations and then focus on fine-tuning in the subsequent 10 simulations. 
We set the policy entropy threshold $\delta$ to 2.
We implement neural network models with PyTorch and decide reasonable values for hyperparameters following the guide of related studies~\cite{rl-iclr-2022-ppo-tricks,icml-2023-vrp-generalization,wang2021variable,huiguo-mm-nft,dl-adam,co-tsp-generalization}. 
See Appendix~\ref{appendix:implementation-details} for hyperparameter settings on neural networks and meta-RL.

\textbf{Baselines}. To validate the effectiveness of FlagVNE, we compare it with the following SOTA heuristics (NRM-VNE \cite{vne-iotj-2018-nrm}; NEA-VNE \cite{vne-tpds-2023-nea}; PSO-VNE \cite{vne-book-2021-pso}) and RL-based baselines (MCTS-VNE \cite{vne-tcyb-2017-mcts}; PG-CNN \cite{vne-tnsm-2022-pg-cnn}; A3C-GCN \cite{vne-tsc-2023-gcn-mask}; DDPG-Attention \cite{vne-tpds-2023-att}). See Appendix~\ref{appendix:baselines} for their descriptions.

\textbf{Metrics}. The following metrics are widely used to evaluate the long-term operational status of network systems over a period $\mathcal{T}$~\cite{vne-survey}: \textit{request acceptance rate} (RAC), \textit{long-term average revenue} (LAR) and \textit{long-term revenue-to-cost} (LT-R2C). See Appendix~\ref{appendix:baselines} for their definitions.

\subsection{Results and Analysis}

\textbf{Overall Performance.} To simulate diverse and complex scenarios with varying traffic throughputs, we manipulate the arrival rate of VNRs in two settings due to the difference in physical resource capacity: in GEANT, we explore a range of [0.001, 0.006] with a step of 0.001, and in WX100, we investigate a range of [0.08, 0.18] stepped by 0.02.

Fig.~\ref{fig:traffic-throughput-test}(a)(b)(c) and (d)(e)(f) illustrate the performance of all algorithms in GEANT and WX100, respectively. As the arrival rate $\eta$ increases, all algorithms experience a decline in RAC on both topologies, attributed to heightened competition for limited physical resources among VNRs. Despite the variability in algorithm performance across different network topologies, influenced by the varying abundance of physical bandwidth resources, FlagVNE consistently achieves the best performance in all scenarios. We observe that the improvements of FlagVNE are more pronounced at higher values of $\eta$, corresponding to heightened resource competition.  This underscores the importance of searchability and generalizability in network environments with limited resources. Specifically, at $\eta = 0.006$ on GEANT, FlagVNE surpasses A3C-GCN, NEA-VNE and NRM-VNE by margins of 10.4\%, 20.7\%  and 27.9\% on RAC, 10.5\%, 28.1\%, and 44.2\% on LAR, and 12.8\%, 28.4\%, and 45.1\% on LT-R2C. On WX100, compared to A3C-GCN, NEA-VNE and NRM-VNE, FlagVNE shows average improvements over different $\eta$ of 12.4\%, 12.5\% and 17.4\% in RAC, 12.8\%, 10.4\% and 24.3\% on LAR, and 9.1\%, 6.7\% and 36.7\% on LT-R2C, respectively. Overall, FlagVNE demonstrates exceptional performance across various network system conditions.

\textbf{Ablation Study.}
To verify the effectiveness of each proposed component, we build several variations of FlagVNE: 
(1) \textit{FlagVNE-UniActionNEA} replaces the bidirectional action with the unidirectional one and sorts the decision sequence of virtual nodes with NEA~\cite{vne-tpds-2023-nea}. (2) \textit{FlagVNE-MetaFree-SinglePolicy} trains a single general policy with valina PPO, without the help of Meta-RL. (3) \textit{FlagVNE-MetaFree-MultiPolicy} directly trains a set of sub-policies from scratch, without using Meta-RL. (4) \textit{FlagVNE-MetaPolicy} only uses the meta-policy to handle variable-sized VNRs. (5) \textit{FlagVNE-NoCurriculum } discards the curriculum scheduling strategy during the meta-learning process.

We examine their performance under arrival rate settings of $\eta=0.006$ on GEANT and $\eta=0.18$ on WX100. These cases exhibit more intense competition for resources, accentuating the performance differentials stemming from the algorithms' searchability and generalizability.
As shown in Table~\ref{table:ablation-study}, FlagVNE outperforms all variations on three metrics, demonstrating that each component of FlagVNE contributes to the improvement in the final performance. Notably, we observe significant performance declines in FlagVNE-MetaFree-MultiPolicy and FlagVNE-MetaFree-SinglePolicy compared to FlagVNE, which shows the effectiveness of our meta-RL training method with a curriculum scheduling strategy in achieving generalization.


\textbf{Additional Evaluation.} Due to the space limit, we place more experiments and analyses in Appendix~\ref{appendix:additional-validation}. Concretely, in Appendix~\ref{appendix:running-time-test}, we conduct the running time test to verify the efficiency of FlagVNE, and the results show that FlagVNE strikes a better balance between performance and running time. 
We also provide the adaptation and convergence analysis in Appendix~\ref{appendix:convergence}. Results show that FlagVNE can efficiently learn a meta-policy with cross-size knowledge in known distributions and quickly adapt to unseen sizes through fine-tuning. Besides, in Appendix~\ref{appendix:large-scale-system-valadition}, we evaluate FlagVNE's scalability on large-scale network systems, and results demonstrate that FlagVNE consistently outperforms all baseline models, even in this large-scale network system scenario. Finally, in Appendix~\ref{appendix:hyperparameter}, we explore the impact of the key hyperparameter $\epsilon$ on the performance of FlagVNE.





\section{Conclusion}
In this paper, we proposed FlagVNE, a flexible and generalizable RL framework for VNE. Specifically, we developed a bidirectional action MDP modeling approach to enable the joint selection of virtual nodes and physical nodes, which expands the agent's search space. Additionally, we designed a hierarchical recorder with a bilevel policy to ensure adaptive output and high training efficiency. Furthermore, we presented a generalizable training method based on meta-RL that efficiently trains a set of size-specific policies to tackle VNRs of varying scales. We also developed a curriculum scheduling strategy that gradually incorporates larger VNRs, thus alleviating suboptimal convergence. Finally, we conducted extensive experiments to verify the effectiveness of FlagVNE.

\clearpage
\section*{Acknowledgements}
This work was partially supported by National Natural Science Foundation of China (Grant No.92370204), Guangzhou-HKUST(GZ) Joint Funding Program (Grant No.2023A03J0008), Education Bureau of Guangzhou Municipality, Guangdong Science and Technology Department Project funded by China Postdoctoral Science Foundation (Grant No.2023M730785), National Natural Science Foundation of China (Grant No. 62102053).

\bibliographystyle{named}
\bibliography{ijcai24}

\clearpage
\clearpage
\appendix
\newtheorem*{theorem*}{Theorem}

\section{Preliminary Study}
\label{appendix:preliminary-study}
\begin{figure}[t]
    \centering
    \includegraphics[width=.48\textwidth]{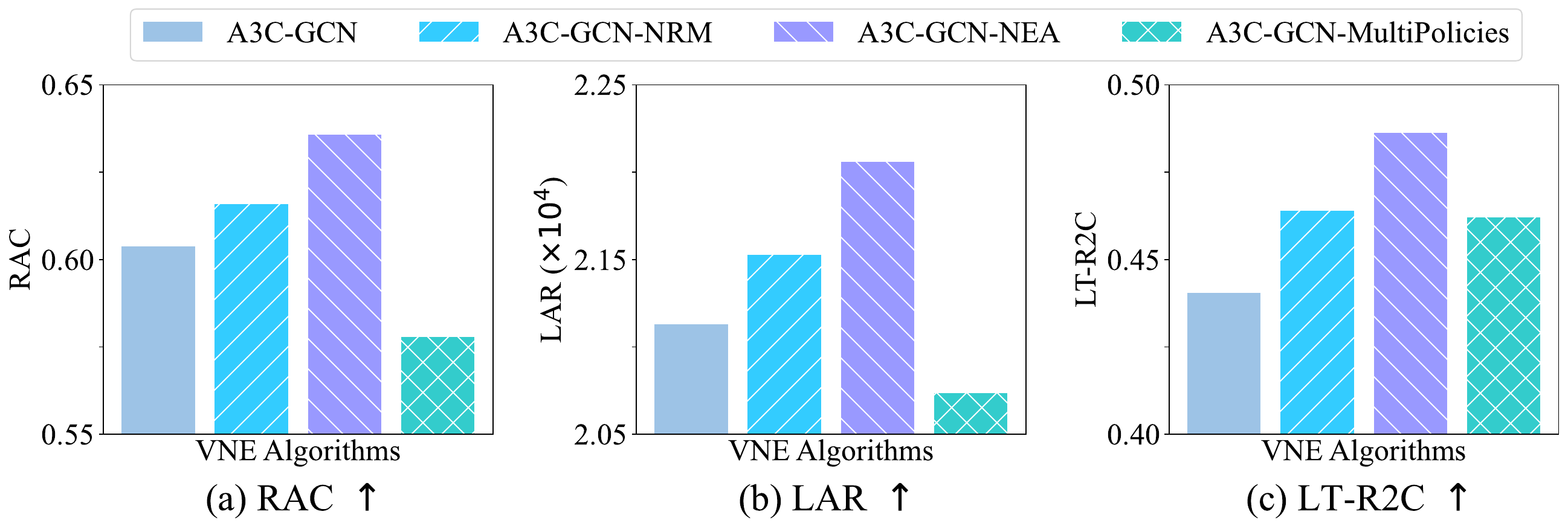}
    \vspace{-1em}
    \caption{Comparative performance of A3C-GCN variants on three metrics: Impact of decision sequence and size-specific policies on VNE. (We conduct experiments using WX100 as the physical network, with a VNR arrival rate of 0.18. All other settings remained consistent with those described in Section~\ref{section:evaluation}.)} 
    \label{fig:preliminary-study}
    \vspace{-1.4em}
\end{figure}

\label{section:motivation}
We conduct a preliminary analysis to reveal the limitations of existing RL-based works and highlight our motivations and latent challenges.
Concretely, we investigate a widely adopted RL-based VNE algorithm, A3C-GCN, used in~\cite{vne-jsac-2020-a3c-gcn,vne-tsc-2023-gcn-mask}. We extend its capabilities by incorporating several methods. Instead of using the default decision sequence based on virtual node ID numbers, we introduce two variants: A3C-GCN-NRM and A3C-GCN-NEA. These variants utilize NRM~\cite{vne-iotj-2018-nrm} and NEA~\cite{vne-tpds-2023-nea} metrics, respectively, to rank virtual nodes and rearrange the decision sequence. 
Furthermore, we develop A3C-GCN-MultiPolicy that fine-tune the pretrained A3C-GCN model to obtain multiple policies tailored to different VNR sizes, rather than use a single one-size-fits-all to accommodate all sizes.
Through the analysis of experimental results, our motivations are attributed to the following two key aspects.

\subsection{The Flexibility of Action Space}
\label{appendix:flexibility-of-action-space}
Most existing approaches based on RL suffer from a unidirectional action design, assuming that the decision sequence of virtual nodes is fixed. They severely limit the action search space from $|\mathcal{N}^p| \times |\mathcal{N}^v|$ to $|\mathcal{N}^p| \times 1$, where $|\mathcal{N}^p|$ and $|\mathcal{N}^v|$ denote the number of physical nodes and virtual nodes, respectively.
However, as shown in Fig.~\ref{fig:preliminary-study}, A3C-GCN, A3C-GCN-NRM and A3C-GCN-NEA exhibit different performance due to their distinct decision sequence of virtual nodes, which highlights the importance of sufficiently exploring different decision sequences to achieve better solutions.
Moreover, the fixed decision sequence of virtual nodes fails to perceive the dynamic nature of the environment and the potential inter-dependencies between decisions made at distinct time steps.
Therefore, due to this action design, the exploration and exploitation of the solution space by the RL agent are inflexible, which restricts the searchability of the algorithm and reduces the probability of discovering high-quality solutions.

To enhance exploration and exploitation and consider the dynamic nature of solving process,
an accessible way is to achieve a joint selection of physical and virtual nodes to eliminate the fixed decision sequence.
But it will pose some challenges, such as the difficulty of variable action probability distribution generation and the training efficiency issue caused by large action space.

\subsection{The Generalization of Solving Policy}
\label{appendix:generalization-of-solving-policy}

\begin{figure}[t]
    \centering
    \includegraphics[width=.40\textwidth]{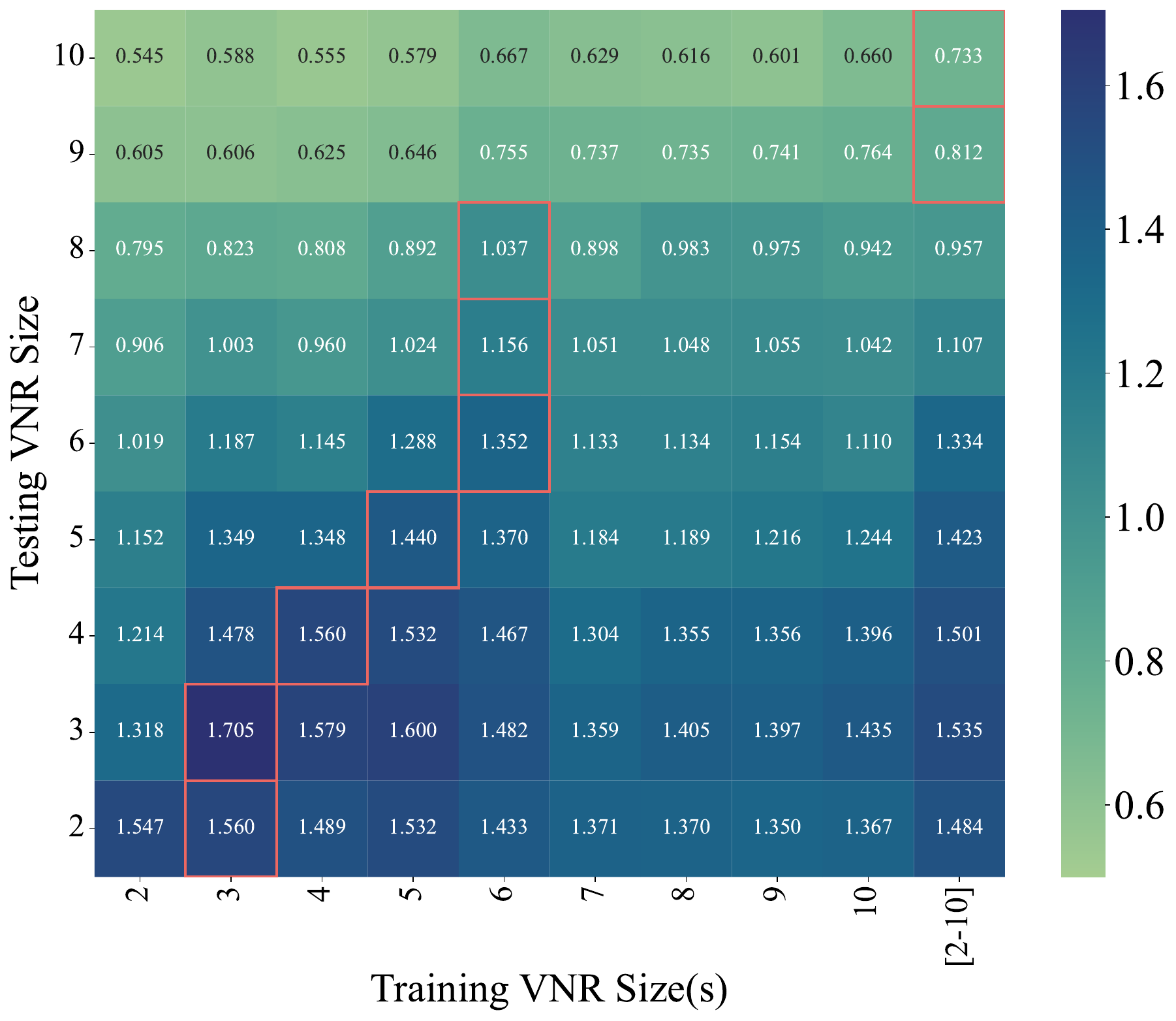}
    \vspace{-1em}
    \caption{Average returns of the one-fits-all-size policy and each size-specific policy on all testing VNR sizes. The red boxes indicate the best performance results for test sizes. In the horizontal axis, [2-10] indicates a well-trained A3C-GCN policy while a single number represents a size-specific policy derived from well-trained A3C-GCN-MultiPolicy. (We use WX100 as the physical network and all training settings are the same as those mentioned in Section~\ref{section:evaluation}. For testing data of each VNR size, to exclude network system dynamics for a fairer comparison, we randomly generated 1000 static instances, including VNR and physical networks, as the benchmark. The performance metric is defined as the average episode return over 1000 instances.)} 
    \label{fig:generalization-of-solving-policy}
    \vspace{-1.4em}
\end{figure}

As revealed in studies~\cite{co-tsp-generalization,icml-2023-vrp-generalization}, 
combinatorial optimization problems on graphs demonstrate variable complexity as graph sizes change. A one-size-fits-all strategy for handling these instances across different sizes often leads to generalization issues. VNE is no exception. Existing methods ignore this diversity and train a universal solving policy with vanilla RL approaches, which fails to effectively balance the strategic knowledge required for varying sizes of VNRs. Additionally, training within a limited range of VNR sizes also incurs a deficiency in generalization and adaptability for previously unseen VNR sizes. Furthermore, training a universal solving policy may perform reasonably well on average but not be optimal for specific VNR sizes. This lack of specialization for different sizes results in suboptimal solutions when dealing with VNRs of varying sizes. Therefore, addressing this generalization problem is crucial to achieving high-quality solutions.

One intuitive solution is to train a set of sub-policies directly to handle VNRs of different sizes. However, as shown in Fig.~\ref{fig:preliminary-study}, A3C-GCN-MultiPolicy that have multiple size-specific policies are even inferior to A3C-GCN on some performance metrics. To further explore this phenomenon, we test the performance of each size-specific strategy on different sizes and provide the results in Fig.~\ref{fig:generalization-of-solving-policy}. We can observe that specific policies trained for VNRs of small sizes often achieve the best performance in their corresponding testing dataset of the same size. Meanwhile, the specific policies of large-sized VNRs exhibit unsatisfactory performance on nearly all testing sizes.  Even on the same test size as the training size, their performance is worse than the general policy. This indicates they are stuck in local optima.

This issue arises due to the more complex solution spaces and stringent constraints associated with larger VNRs, which significantly complicate the exploration of high-reward states and the discovery of feasible solutions. Consequently, when these sub-policies for large-sized VNRs are trained from scratch, they tend to settle into suboptimal solutions and lead to diminished performance. Furthermore, this method based on direct training from scratch lacks quick adaptability to handle situations where unseen VNR sizes occur. When faced with new sizes, the system needs to collect ample data to train policies for the new sizes from scratch, which significantly increases the demand for data.

\section{Problem Formulation}
\label{appendix:problem-formulation}
\subsection{Optimization Objectives}
Acknowledging the stochastic nature of online networking, most existing methods and this work aim to minimize the embedding cost of each arriving VNR onto the physical network. This facilitates the improvement of resource utilization and VNR acceptance rate. To evaluate the quality of solutions, the revenue-to-cost ratio (R2C) serves as a crucial indicator, defined as follows:
\begin{equation}
    \text{R2C} \left(\mathcal{G}^v\right) = \left(\Psi \cdot \text{REV} \left(\mathcal{G}^v\right) \right) / \text{COST}\left(\mathcal{G}^v\right).
\end{equation}

Here, $\Psi$ is the binary variable that indicates the feasibility of a solution. $\Psi = 1$ if the solution of $G^v$ is accepted, and $\Psi = 0$ if it is not. $\text{REV} (\mathcal{G}^v)$ denotes the revenue of the VNR $G^v$ and $\text{COST} (\mathcal{G}^v)$ denotes the embedding cost resulting from the solution, which are calculated as follows:
\begin{equation}
\text{REV} (\mathcal{G}^v) = \frac{1}{|\mathcal{C}|} \sum_{n^v \in \mathcal{N}^v} \sum_{C \in \mathcal{C}} C (n^v) + \sum_{l^v \in \mathcal{L}^v} B (l^v),
\end{equation} 
\begin{equation}
    \text{COST} (\mathcal{G}^v) = \frac{1}{|\mathcal{C}|} \sum_{n^v \in \mathcal{N}^v} \sum_{C \in \mathcal{C}} C (n^v) + \sum_{l^v \in \mathcal{L}^v} H_{l^v} B (l^v),
\end{equation}
where $H_{l^v}$ denotes the hop count of the physical path that routes the virtual link $l^v$.

\subsection{Constraint Conditions}
The process of embedding one VNR $\mathcal{G}^v \in \mathcal{V}$ onto the physical network is formulated as a mapping function $f: \mathcal{G}^v \rightarrow \mathcal{G}^p$. 
This process utilizes two types of boolean variables to make decisions: (1) $x^{m}_{i} = 1$ indicating that virtual node $n^v_m$ is placed in physical node $n^p_i$, and 0 otherwise; (2) $y^{m,w}_{i,j} = 1$ indicating that virtual link $l^v_{m,w} = ({n^v_m}, {n^v_w})$ traverses through physical link $l^p_{i,j} = ({n^p_i}$, ${n^p_j})$, and 0 otherwise. Here, we use the $m$ and $w$ as identifiers of physical nodes and $i$ and $j$ as identifiers of virtual nodes. A VNR is successfully embedded when a feasible mapping solution is found, i.e., the following constraints are satisfied:
\begin{equation}
  \sum_{n^p_i \in \mathcal{N}^p} x^{m}_{i} = 1, \forall n^v_m \in \mathcal{N}^v,
  \label{eq:vne-nm-1}
\end{equation}
\begin{equation}
  \sum_{n^v_m \in \mathcal{N}^v} x^{m}_{i} \leq 1, \forall n^p_i \in \mathcal{N}^p,
  \label{eq:vne-nm-2}
\end{equation}
\begin{equation}
\begin{split}
    x^{m}_{i} C(n^v_m) \leq C(n^p_i),  \forall n^v_m \in \mathcal{N}^v, n^p_i \in \mathcal{N}^p, C \in \mathcal{C},
\end{split}
\label{eq:vne-nm-3}
\end{equation}
\begin{equation}
\begin{split}
\sum_{\!\!\!n^p_i \in \Omega (n^p_k) }\!\!\! y^{m,w}_{i,k} \!-\!\!\! \sum_{\!\!\!n^p_j \in \Omega (n^p_k) } \!\!\! y^{m,w}_{k,j} \!= x^{m}_{k} \!-\! x^{w}_{k},
\forall l^v_{m, w} \!\!\in\! \mathcal{L}^v, n^v_k \in\! \mathcal{N}^p,
\end{split}
\label{eq:vne-lm-1}
\end{equation}
\begin{equation}
\begin{split}
y^{m,w}_{i,j} + y^{m,w}_{j,w} \leq 1, 
\forall l^v_{m, w} \in \mathcal{L}^v, l^p_{i, j} \in \mathcal{L}^p,
\end{split}
\label{eq:vne-lm-2}
\end{equation}
\begin{equation}
\begin{split}
  \!\!\!\sum_{l^v_{m, w} \in \mathcal{L}^v} (y^{m,w}_{i,j} \!+\! y^{m,w}_{j,i}) B(l^v_{m, w}) \leq B((l^p_{i, j})),
  \forall (l^p_{i, j}) \in \mathcal{L}^p.
\end{split}
\label{eq:vne-lm-3}
\end{equation}

Here, $\Omega (n^p_k)$ denotes the neighbors of $n^p_k$. Constraint \eqref{eq:vne-nm-1} ensures that each virtual node is assigned to exactly one physical node, while the constraint \eqref{eq:vne-nm-2} ensures that each physical node accommodates at most one virtual node, thus enforcing a one-to-one mapping between them. Constraint \eqref{eq:vne-nm-3} ensures that each virtual node is assigned to a physical node with sufficient resources where available resources of all types exceed the demands. Constraint \eqref{eq:vne-lm-1} ensures that each virtual link $(n^v_m, n^v_w)$ is routed through a connective physical path from the physical node $n^p_i$ (where virtual node $n^v_m$ is mapped, i.e., $x^m_i = 1$) to the physical node $n^p_j$ (where virtual node $n^v_w$ is mapped, i.e., $x^w_j = 1$). This constraint follows the flow conservation law. Constraint \eqref{eq:vne-lm-2} prevents routing loops and guarantees acyclic routing of physical paths for virtual links. Constraint \eqref{eq:vne-lm-3} ensures that the bandwidth consumption of each physical link does not exceed its available capacity.

\section{Proof of the Superiority of MDP Optimality}
\label{appendix:proof}
\begin{theorem*}
Given two MDPs with the bidirectional and unidirectional action, $\mathcal{M}^b = \langle \mathcal{S}^b, \mathcal{A}^b, P, R, \lambda\rangle$ and $\mathcal{M}^u = \langle\mathcal{S}^u, \mathcal{A}^u, {P}, {R}, \lambda\rangle$, and their optimal policy denote as $\pi^{\star, b}$ and $\pi^{\star, u}$, respectively, we have $\pi^{\star, b} \succeq \pi^{\star, u}$. 
\end{theorem*}

\begin{proof}
In $\mathcal{M}^b$, the bidirectional action at the timestep $t$, $a_t = (n^p_t, n^v_t)$, is the joint selection of virtual and physical nodes. The optimal policy for $\mathcal{M}^b$ can be defined as:

\begin{equation}
    \pi^{\star,b} (a_t | s_t) =   \pi^{H,\star} (n^v_t | s_t) \cdot \pi^{L,\star} (n^p | s_t, n^v_t),
\end{equation}
where $\pi^{H,\star} (n^v_t | s_t)$ optimally orders virtual nodes and $\pi^{L,\star} (n^p | s_t, n^v_t)$ optimally selects physical nodes given the virtual node selection.

In $\mathcal{M}^u$, the predefined policy $\pi^{H,u}(n^v_t | s_t)$ governs the decision sequence of virtual nodes. Similarly, The optimal policy for $\mathcal{M}^u$ can be similarly defined as
\begin{equation}
    \pi^{\star,u} (a_t | s_t) =   \pi^{H,u} (n^v_t | s_t) \cdot \pi^{L,\star} (n^p | s_t, n^v_t).
\end{equation}
Notably, $\pi^{H,u}$ may not be optimal.

Following the standard reinforcement learning framework~\cite{book-2018-rl-introduction},  the value of a state $s$ under a policy $\pi$, denoted $V_{\pi}(s)$ is the expected return when starting in $s$ and following $\pi$ thereafter. For our bilevel policy, $V_{\pi}(s)$ can be formulated as:
\begin{equation}
    V_{\pi}(s) = \mathbb{E}_{n^v \sim \pi^H, n^p \sim \pi^L} [ \sum_{t=0}^{T} \lambda^t R(s_t, a_t) | s_0 = s].
\end{equation}
If and only if  $V_{\pi}(s) \ge V_{\pi^{\prime}}(s)$ for all state, $\pi \succeq \pi^{\prime}$.

Since $\pi^{H,u}$ in $\pi^{\star,u}$ is not necessarily optimal, for some states, it might not maximize the expected return, i.e., $\pi^{H,u}(n^v | s_t) \neq \pi^{H,\star}(n^v | s_t)$.
In contrast, $\pi^{H,\star}$ in $\pi^{\star,b}$ is optimal by definition. Thus, for some states, we have 
\begin{equation}
    V_{\pi^{\star,b}}(s) > V_{\pi^{\star,u}}(s).
\end{equation}
This inequality demonstrates the potential sub-optimality of $\pi^{\star,u}$ compared to $\pi^{\star,b}$.

Given the above, we conclude that $\pi^{\star, b} \succeq \pi^{\star, u}$.
\end{proof}

\section{Pseudocode of FlagVNE Training}
\label{appendix-pseudocode}

We describe FlagVNE's training process in Algorithm~\ref{algo:flagvne-training}.

\begin{algorithm}[t]
    \small
    \caption{\small Training Process of FlagVNE }
    \label{algo:flagvne-training} 
    \Input{Initial meta-policy $\phi$; Policy set $\Theta = \{\phi\}$;\\ Meta learning rate $\beta$;
     Task learning rate $\alpha$; \\
     Policy entropy threshold $\delta$
    }
    \Output{Trained policies set $\Theta$;}
    \BlankLine
    \tcp{Meta-learning Process}
    Initialize the training task ID list $\mathcal{I} = \{1\}$;\\
    \While{not done}
    {
       Collect the trajectory memory $\mathcal{D}$ by interactions;\\
        Split $\mathcal{D}$ into $\{\mathcal{D}_1,\cdots,\mathcal{D}_{|\mathcal{M}|}\}$ based on VNR' size;\\
        Analyze the task distribution $\mathcal{M}_i \sim p(\mathcal{M})$;\\
        \For{$i = 1, 2, \cdots, |\mathcal{M}|$}{
            \If{$i \notin \mathcal{I}$}{
                \Continue
            }
            $\theta_i \leftarrow \text{DeepCopy}(\phi)$;\\
            Update $\theta_i$ with Eq. \eqref{eq:maml-inner} and \eqref{eq:ppo}; \textit{// Inner loop}\\
        }
        Update $\phi$ with Eq. \eqref{eq:maml-outer}; \textit{// Outer loop}\\
        \tcp{Curriculum Scheduling Strategy}
        Get the current most complex task ID $k = \mathrm{max}(\mathcal{I})$;\\
        \If{$H(\pi_{\theta_k}) < \delta$ and $k < |\mathcal{M}|$}{
            $\mathcal{I} \leftarrow \mathcal{I} \cup \{k + 1\}$
        }
    }
    \tcp{Fine-tuning Process}
    \For{$i = 1, 2, \cdots, |\mathcal{M}|$}{
        $\theta_i \leftarrow \text{DeepCopy}(\phi)$;\\
    }
    \While{not done}
    {
    Collect task trajectory memories $\{\mathcal{D}_1,\cdots,\mathcal{D}_{|\mathcal{M}|}\}$;\\
    \For{$i = 1, 2, \cdots, |\mathcal{M}|$}{
        Update $\theta_i$ with Eq. \eqref{eq:maml-inner} and \eqref{eq:ppo}; \textit{// Inner loop}\\
    }
    }
    \For{$i = 1, 2, \cdots, |\mathcal{M}|$}{
        $\Theta \leftarrow \Theta \cup \{ \theta_i \};$
    }
    
\end{algorithm}

\section{Experimental Details}
\label{appendix:experimental-details}

\subsection{Topology Descriptions}
\label{appendix:topology}
The descriptions of used realistic topologies are as follows:
\begin{itemize}
    \item \textit{GEANT} is a network infrastructure linking Europe's national research and education networks. It acts as a pan-European backbone, connecting European researchers, academics, and students, and extends globally to over half the world's countries. This network topology has 40 nodes and 64 lines, with a density of 0.0821.
    \item \textit{WX100} is generated following the Waxman graph method~\cite{dataset-jsac-1988-waxman}. The Waxman graph is a type of random graph model, widely used in the field of communication and network simulation. WX100 has 100 nodes and 500 links, with a density of 0.1010.
\end{itemize}

\subsection{Implementation Details}
\label{appendix:implementation-details}
Detailed hyperparameter settings are as follows: each neural network has a hidden dimension $H$ of 128 and the GNN module consists of $K = 3$ GCN layers. We train this model using Adam optimizer~\cite{dl-adam}, with learning rates $\alpha$ and $\beta$ of 0.001 and a batch size of 128. In PPO, the RL discounted factor $\lambda$ is set to 0.99, the coefficient of critic loss is set to 0.5, and the number of repeat times is set to 10. All experiments are carried out on a server equipped with 24 Intel Xeon E5-2650 v4 @ 2.20GHz CPUs and 4 V100 GPUs. 

\subsection{Baseline Descriptions}
\label{appendix:baselines}
The descriptions of compared baselines are as follows:
\begin{itemize}
    \item NRM-VNE~\cite{vne-iotj-2018-nrm} sorts virtual and physical nodes based on a node resource management metric. 
    Greedy matching and breadth-first search are then used for node mapping and link mapping, respectively.
    \item NEA-VNE~\cite{vne-tpds-2023-nea} ranks virtual and physical nodes with an essentiality assessment metric and then performs node mapping and link mapping similar to~\cite{vne-iotj-2018-nrm}.
    \item PSO-VNE~\cite{vne-book-2021-pso} uses particle swarm optimization (PSO) to explore the VNE's solution space.
    \item MCTS-VNE~\cite{vne-tcyb-2017-mcts} uses the Monte Carlo tree search algorithm to search solutions.
    \item PG-CNN~\cite{vne-tnsm-2022-pg-cnn} builds a policy network with convolutional neural network (CNN) and trains it with PG algorithm.
    \item A3C-GCN~\cite{vne-tsc-2023-gcn-mask} constructs a policy network with GCN and MLP and uses A3C algorithm as the training method. It extends the method proposed in \cite{vne-jsac-2020-a3c-gcn} with an action mask mechanism.
    \item DDPG-Attention~\cite{vne-tpds-2023-att} builds an attention-based policy and trains it with DDPG algorithm.
\end{itemize}

\subsection{Performance Metrics}
\label{appendix:metrics}
The following metrics are widely used to evaluate the long-term operational status of the network system over a period $\mathcal{T}$~\cite{vne-survey,vne-jsac-2020-a3c-gcn}:

\textit{Request Acceptance Rate} (RAC) quantifies the ratio of accepted VNRs to the total number of VNRs arrived, indicating the system’s ability to meet user service demands, defined as:
\begin{equation}
\text{RAC} = \frac{\sum^\mathcal{T}_{t=0} |\mathcal{V}_s(t)|}{\sum^\mathcal{T}_{t=0} |\mathcal{V}(t)|},
\end{equation}
where $\mathcal{V}(t)$ and $\mathcal{V}_s(t)$ denote the set of totally arrived and accepted VNRs at the unit of time slot $t$, respectively. 

\textit{Long-term Average Revenue} (LAR) quantifies the total gained revenue, directly assessing the financial performance of the ISP, defined as:
\begin{equation}
\text{LAR} = \left( \sum^\mathcal{T}_{t=0} \sum_{\mathcal{G}^v \in {\mathcal{V}}_s(t)} \text{REV} (\mathcal{G}^v) \times d^v \right) / \mathcal{T}.
\end{equation}
where $d^v$ denotes the lifetime of VNR $\mathcal{G}^v$.
\textit{Long-term R2C} (LT-R2C) quantifies the overall revenue-to-cost ratio, evaluating the solution quality of all accepted VNRs and resource utilization, defined as:
\begin{equation}
\text{LT-R2C} = \frac{\sum^\mathcal{T}_{t=0} \sum_{\mathcal{G}^v \in {\mathcal{V}}_s(t)} \text{REV}(\mathcal{G}^v) \times d^v}
{\sum^\mathcal{T}_{t=0} \sum_{\mathcal{G}^v \in {\mathcal{V}}_s(t)} \text{COST}(\mathcal{G}^v) \times d^v}.
\end{equation}

\section{Additional Validation}
\label{appendix:additional-validation}

\subsection{Running Time Test}
\label{appendix:running-time-test}
Due to the low-latency requirements of network systems, VNE algorithms should provide solutions within an acceptable time. The average running time over different traffic throughputs same to the above settings of each algorithm is illustrated in Table~\ref{table:avg-running-time}. Compared to PSO-VNE and MCTS-VNE which are quite time-consuming, FlagVNE and other algorithms can offer solutions more efficiently. While NRM-VNE achieves the fastest solving speeds, it is noteworthy that FlagVNE markedly outperforms NRM-VNE on evaluation metrics. Overall, FlagVNE strikes a better balance between performance and running time.

\begin{table}[]
\centering
\begin{threeparttable}
\begin{tabular}{c|cc}
\toprule
& \multicolumn{2}{c}{Average Running Time (s) $\downarrow$ }  \\ \cline{2-3}
& \quad GEANT \quad & \quad WX100 \quad \\ 
\midrule
NRM-VNE & 10.079 & 28.079 \\
NEA-VNE & 31.011 & 238.403 \\
PSO-VNE & 1330.706 & 1516.340\\
MCTS-VNE & 240.195 & 679.007 \\
PG-CNN & 75.259 & 203.965 \\
A3C-GCN & 47.079 & 204.073 \\
DDPG-Attention & 81.713 & 164.355 \\
FlagVNE & 84.987 & 239.251 \\ 
\bottomrule
\end{tabular}
\begin{tablenotes}
\small
\item[*] The average simulation time (seconds) over various $\eta$
\end{tablenotes}
\end{threeparttable}
\vspace{-0.6em}
\caption{Average running time in traffic throughput test. }
\label{table:avg-running-time}
\vspace{-1.4em}
\end{table}

\begin{figure}[t]
    \centering
    \includegraphics[width=.45\textwidth]{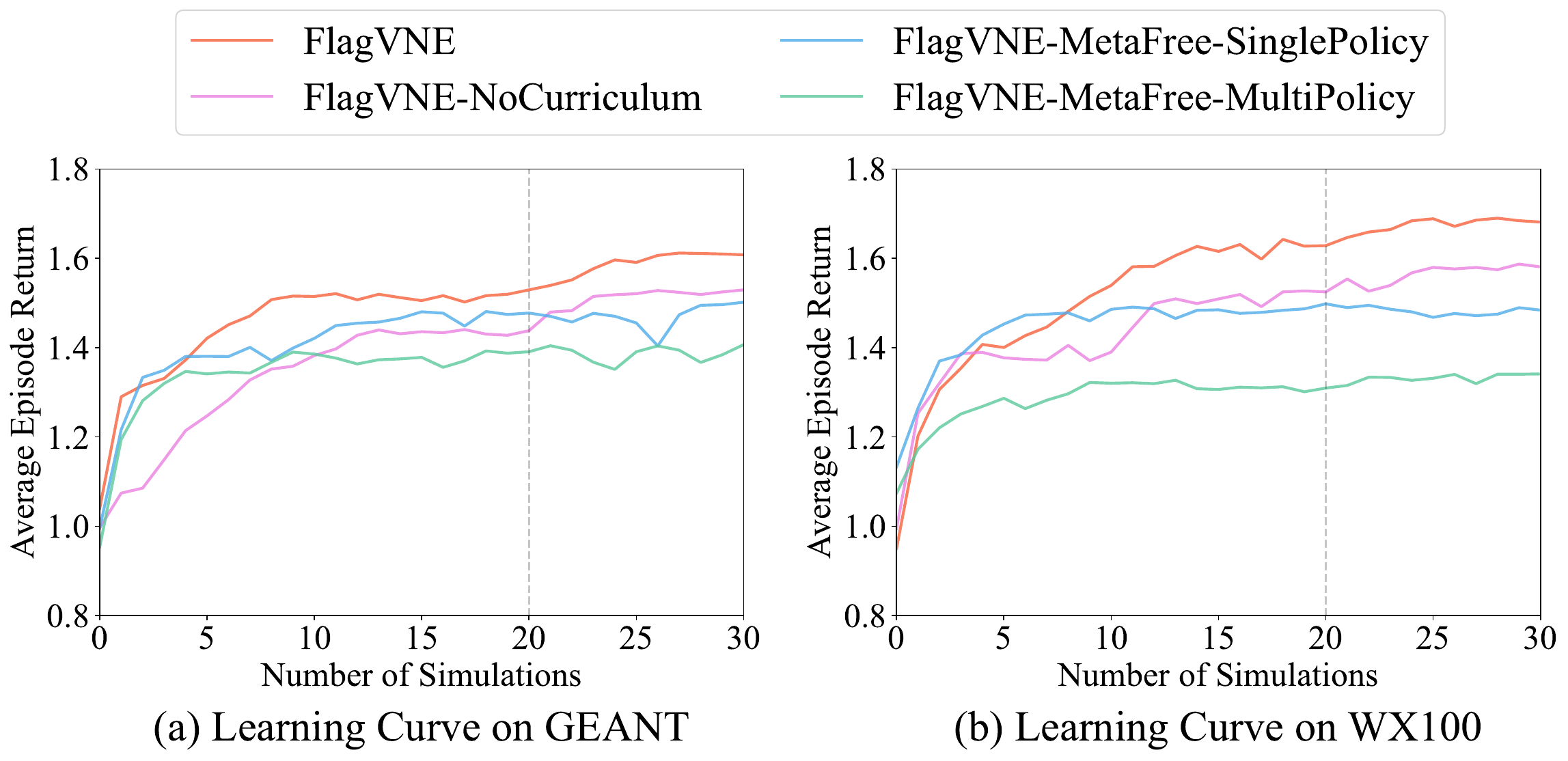}
    \vspace{-1em}
    \caption{Learning curves over simulations. Within each simulation, there are 1000 VNRs, i.e., 1000 episodes and we average their returns.  ($\eta = 0.001$ on GEANT and $\eta = 0.06$ on WX100.)} 
    \label{fig:learning-curves}
    \vspace{-0.6em}
\end{figure}

\subsection{Adaptation and Convergence Analysis}
\label{appendix:convergence}

We provide the analysis of learning curves on both known distribution and unseen size.

\textbf{Training on Known Distribution.}
Fig.~\ref{fig:learning-curves} presents the learning curves of FlagVNE alongside three other variations, trained on a known distribution of VNR's size [2-10]. We observe that FlagVNE achieves higher average returns at convergence compared to both FlagVNE-MetaFree-SinglePolicy and FlagVNE-MetaFree-MultiPolicy. This outcome highlights the significance of generalization in training methods and validates the efficacy of our generalizable training approach. Additionally, FlagVNE's superior performance over FlagVNE-NoCurriculum underscores the effectiveness of our curriculum scheduling strategy in mitigating suboptimal convergence and enhancing the overall learning process.

\textbf{Adaptation to Unseen Size.}
\label{appendix:adaptation-to-unseen-size}
\begin{figure}[t]
    \centering
    \includegraphics[width=.45\textwidth]{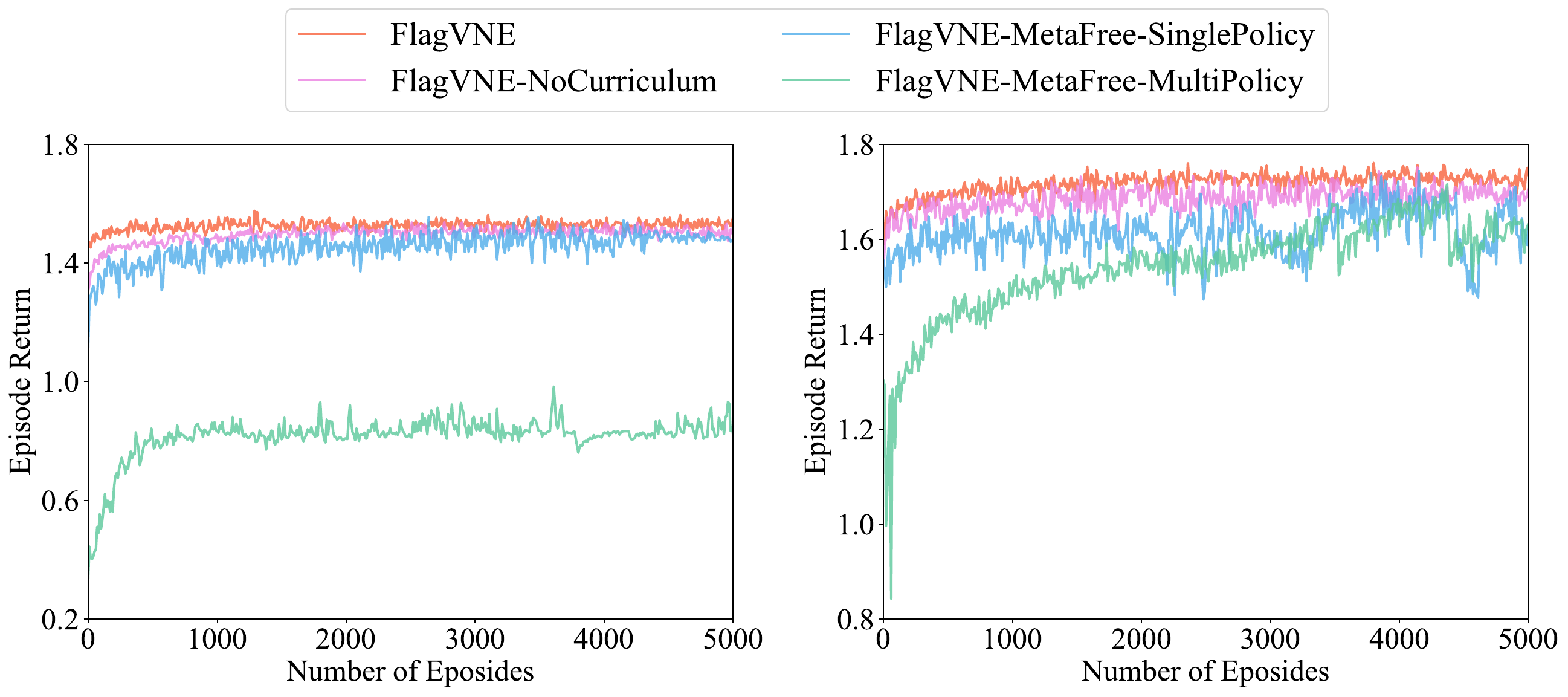}
    \vspace{-1em}
    \caption{Learning curves for the unseen size. (There are 100 VNRs in one simulation and each VNR's size is set to 12).} 
    \label{fig:adaptation-to-unseen-task}
    \vspace{-1.0em}
\end{figure}
To assess FlagVNE's adaptability to previously unseen VNR sizes, we conduct experiments with all incoming VNRs in the network system set to size 12, a size not included in earlier training phases. Given the heightened resource demands of these larger-sized VNRs, we limited the number of simultaneous VNRs to 1 per simulation, maintaining all other simulation parameters as described in Section~\ref{section:experiment-setup}. During this phase, we fine-tuned the meta-policies of both FlagVNE and FlagVNE-NoCurriculum, as well as FlagVNE-MetaFree-SinglePolicy. Meanwhile, FlagVNE-MetaFree-MultiPolicy underwent training to develop a new size-specific policy from scratch. As shown in Fig.~\ref{fig:adaptation-to-unseen-task}, FlagVNE achieves rapid convergence within a small amount of training data, significantly reducing the data requirements. This highlights the benefits of utilizing the initial meta-policy that acquires cross-task knowledge, which enables FlagVNE to rapidly adapt to unseen-size-specific sub-policies. Notably, in the case of GEANT, FlagVNE-MetaFree-MultiPolicy fails to converge within 5000 epochs. This outcome is mainly due to GEANT's smaller topology and constrained resources, which limit the exploration of feasible solutions. Such conditions pose significant challenges for a policy that lacks prior experience or foundational knowledge, further emphasizing the effectiveness of our meta-learning approach in environments with limited resource scenarios.

\subsection{Scalability Valadition}
\label{appendix:large-scale-system-valadition}

\begin{table}[]
\centering
\caption{Results on Scalability Validation.}
\begin{threeparttable}
\begin{tabular}{c|ccc}
\toprule
Algorithm & RAC $\uparrow$ & LAR ($\times 10^6) \uparrow$ & LT-R2C $\uparrow$ \\ \midrule
NRM-VNE & 0.631 & 0.710772 & 0.507 \\
NEA-VNE & 0.857 & 1.186615 & 0.690 \\
PSO-VNE & 0.805 & 1.042604 & 0.537 \\
MCTS-VNE & 0.782 & 0.968175 & 0.563 \\
PG-CNN & 0.851 & 1.046523  & 0.548 \\
A3C-GCN & 0.869 & 1.147116  & 0.715 \\
DDPG-Attention & 0.796 &  1.013670  & 0.617 \\ \midrule
\textbf{FlagVNE} & \textbf{0.932} & \textbf{1.347162} & \textbf{0.744} \\ \bottomrule
\end{tabular}
\end{threeparttable}
\label{table:scalability-validation}
\vspace{-1.4em}
\end{table}

To further validate the effectiveness of FlagVNE in large-scale network systems, we study the performance of all algorithms in large-scale networks. Following previous works~\cite{vne-jsac-2020-a3c-gcn,kdd-2023-gal-vne}, we generate a random Waxman topology with 500 nodes and nearly 13000 links, named WX500~\cite{dataset-jsac-1988-waxman}. We also increase the VNE size distribution to a uniform distribution from 2 to 20, and the arrival rate of VNR $\eta$ is set to 3. For the training of FlagVNE, we execute the meta-learning in the initial 40 simulations and then conduct the fine-tuning with 10 simulations. All other simulation and hyperparameter settings remain consistent with those outlined in Section~\ref{section:experiment-setup}.
The results of all algorithms are shown in Table~\ref{table:scalability-validation}. We observe that FlagVNE and NRM-VNE achieve the best and worst performance, respectively. Compared to other baselines, FlagVNE presents clear performance advantages, which demonstrates its effectiveness in large-scale network systems.

\subsection{Hyperparameter Sensitivity}
\label{appendix:hyperparameter}

\begin{figure}[t]
    \centering
    \includegraphics[width=.48\textwidth]{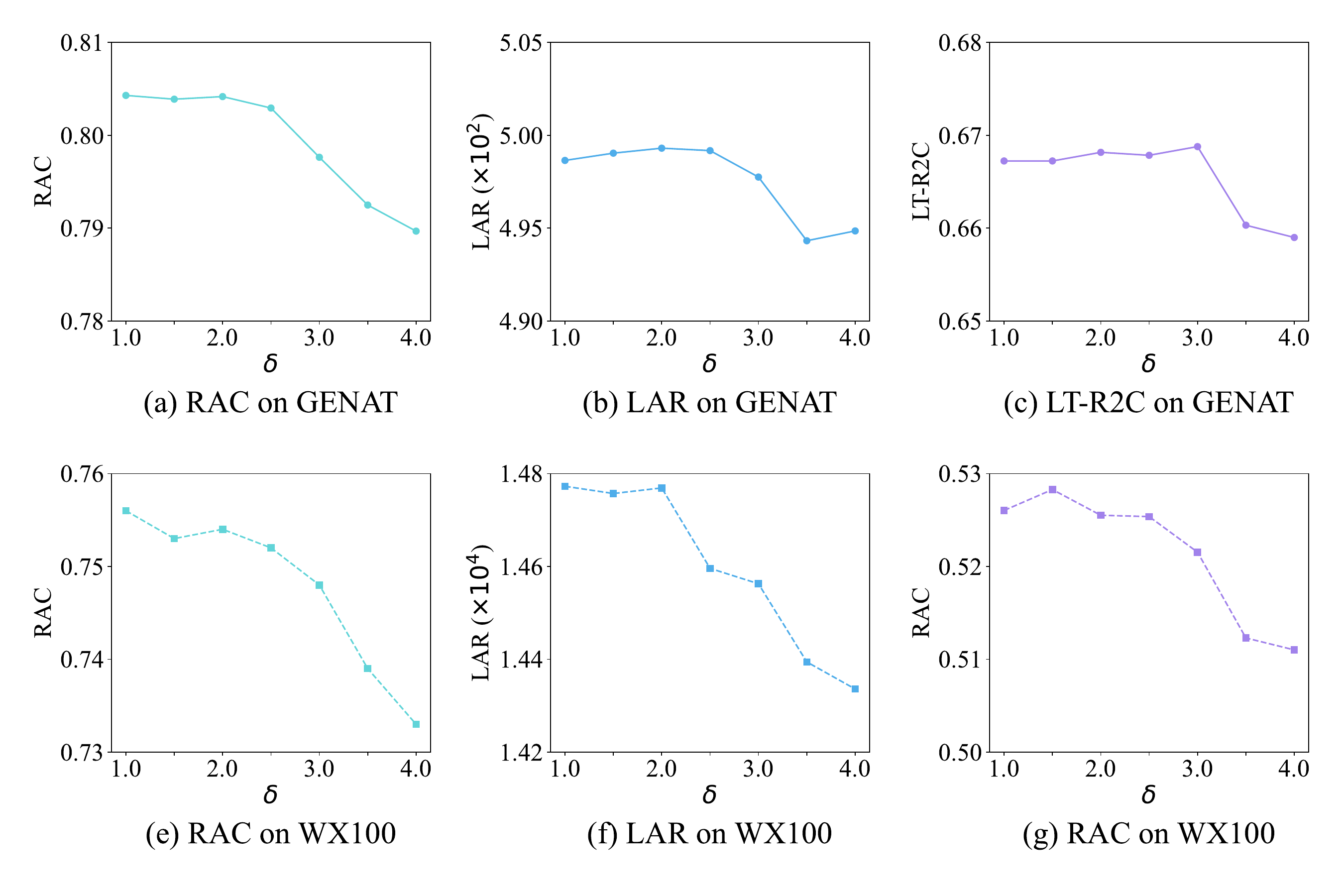}
    \vspace{-1.6em}
    \caption{The impact of $\delta$ on FlagVNE's performance. ($\eta$ = 0.001 on GEANT and $\eta$ = 0.06 on WX100.)} 
    \label{fig:hyperparameter-sensitivity}
    \vspace{-1.4em}
\end{figure}

We investigate the impact of the policy stability threshold $\delta$ on FlagVNE's performance. We keep all other simulation and training parameters consistent with those detailed in Section~\ref{section:experiment-setup}.
The experimental results are shown in Fig.~\ref{fig:hyperparameter-sensitivity}, where (a)(b)(c) and (e)(f)(g) show the testing results on GEANT at $\eta =  0.001$ and WX100 at $\eta =  0.18$, respectively.
We observe that FlagVNE exhibit relative stability within a $\delta$ range of [1.0, 2.5]. 
This stability is attributed to the higher $\delta$ value signals more consistent and confident decision-making by the current meta-policy for the currently largest size VNRs. Such policies can serve as a better initialization of specific policies tailored to larger VNRs. To trade off performance and efficiency, we set this parameter to 2.
Notably, FlagVNE's performance on WX100 demonstrates a higher sensitivity to  $\delta$ compared to GEANT. This increased sensitivity can be ascribed to the more challenging exploration environment in WX100, where the policy is more prone to suboptimal convergence. This observation underscores the value of our curriculum scheduling strategy, which aids in mitigating the issues associated with such exploration challenges by gradually increasing task complexity.

\end{document}